\documentclass[letterpaper]{article}

\usepackage[margin=2.5cm]{geometry}
\usepackage{amsmath,amsthm,amsfonts,amssymb, tikz, array, tabularx, url}
\usepackage[linesnumbered,ruled,vlined]{algorithm2e}
\usetikzlibrary{calc}

\newcommand{\poly}{\mbox{poly}}

\newtheorem{thm}{Theorem}[section]

\newtheorem{prop}[thm]{Proposition}

\newtheorem{thm-algo}[thm]{Algorithm}

\newtheorem{clm}[thm]{Claim}
\newtheorem{defn}[thm]{Definition}

\newtheorem{rmk}[thm]{Remark}
\newtheorem{fact}[thm]{Fact}
\numberwithin{equation}{section}
\numberwithin{figure}{section}

\title{Tight query complexity bounds for learning graph partitions}

\author{
    Xizhi Liu\footnotemark[3]$\,$ \footnotemark[1]
    $\qquad$ 
	Sayan Mukherjee\footnotemark[2]$\,$ \footnotemark[1]
	}

\begin{document}
	\maketitle
	\begin{abstract}
	Given a partition of a graph into connected components, the membership oracle asserts whether any two vertices of the graph lie in the same component or not.
	We prove that for $n\ge k\ge 2$, learning the components of an $n$-vertex hidden graph with $k$ components requires at least $(k-1)n-\binom k2$ membership queries.
	Our result improves on the best known information-theoretic bound of $\Omega(n\log k)$ queries, and exactly matches the query complexity of the algorithm introduced by Reyzin and Srivastava~\cite{reyzin-srivastava-Learning2007} for this problem.
	Additionally, we introduce an oracle, with access to which one can learn the number of components of $G$ in asymptotically fewer queries than learning the full partition, thus answering another question posed by the same authors.
	Lastly, we introduce a more applicable version of this oracle, and prove asymptotically tight bounds of $\widetilde\Theta(m)$ queries for both learning and verifying an $m$-edge hidden graph $G$ using it.
\end{abstract}

\footnotetext[3]{Department of Mathematics, Statistics and Computer Science, University of Illinois at Chicago, USA.}
\footnotetext[3]{Mathematics Institute and DIMAP, University of Warwick, Coventry, CV4 7AL, UK. {\tt Email:xizhi.liu@warwick.ac.uk}}
\footnotetext[2]{Blueqat Research, Shibuya 2-24-12, Tokyo, Japan. {\tt Email:sayan@blueqat.com}.}
\footnotetext[1]{Equal contribution}


{\it Accepted for presentation at the Conference on Learning Theory (COLT) 2022.}
\section{Introduction}

\subsection{Background and Applications}

A graph $G=(V,E)$ consists of a vertex set $V$ and an edge set $E\subseteq \binom V2$.
The field of graph learning deals with learning a hidden graph using queries to black-box oracles that reveal partial information about the graph.
In several real world scenarios, learning a full graph by checking only pairwise adjacency is inefficient, and several oracles can speed up the process of learning by encoding more information per query.
Different oracles could be useful or easier to implement in different scenarios.
For example, in the context of trying to learn a hidden network graph, a \emph{traceroute} query between a pair of vertices can give information about a shortest path between the vertices and their distance in the graph.
In the context of bioinformatics, one can model a graph with vertices corresponding to chemicals, and two vertices are joined by an edge if they react when mixed together.
In such a situation, oracles such as edge-detection and edge-counting can be implemented by mixing different sets of chemicals and measuring the intensity of reaction.
As calling an oracle incurs cost, researchers try to estimate the \emph{query complexity}: the least number of queries to the oracle required to learn a specific graph or graph property.

A separate type of problem that is also widely studied is the problem of graph verification.
In this setting, we have a hidden graph $G=(V,E)$ and a known graph $\widehat G=(V,\widehat E)$, and an oracle that reveals information about $G$.
The main task in this area is to verify whether $G=\widehat G$ using as few oracle queries as possible.
In the context of our paper, we assume $|\widehat E| = |E|$, as $G$ and $\widehat G$ are trivially unequal if they do not have the same number of edges.
Verification tasks are prevalent in real networks where it is important to make sure a recent snapshot of a network is accurate.

Due to their practical and theoretical importance, both graph learning and verification have garnered a lot of interest in recent years.
Perhaps the first problem considered in the literature was the problem of learning a degree-bounded tree using the shortest path oracle \cite{hein1989optimal,king2003complexity}.
The shortest path oracle and the distance oracle were extensively studied in \cite{reyzin-srivastava-Learning2007,kannan2018graph,rong2021reconstruction}.
The current best upper bound on learning a connected bounded-degree graph on $n$ vertices is due to  Mathieu and Zhou~\cite{mathieu2013graph}, where they provide a randomized $\widetilde O(n^{3/2})$-query algorithm to learn such a graph.
Abrahamsen et. al.~\cite{abrahamsen2016graph} prove the exact bound on the query complexity using a weaker oracle called the \emph{betweenness oracle}, thus suggesting that the bound of $\widetilde O(n^{3/2})$ might be not tight.
Parallel results in graph verification have also been obtained for the distance oracle \cite{kannan2015near} and betweenness oracle \cite{janardhanan2017graph}.
Although there is still a gap between the lower bound of $\Omega(n)$ and upper bound of $\widetilde O(n^{3/2})$ for learning degree-bounded graphs using distance oracle, the gap was closed recently for random degree-bounded regular graphs by Mathieu and Zhou~\cite{mathieu2021simple}.

Another well-studied oracle is the \emph{edge-detection oracle}, which, given a set of vertices of the hidden graph $G$, tells if it is an independent set or not.
Results on learning restricted classes of graphs using this oracle such as matchings \cite{alon2004learning}, stars and cliques \cite{alon2005learning}, Hamiltonian cycles \cite{grebinski2000optimal} were succeeded by a very general treatment by  Angluin and Chen~\cite{angluin2008learning}.
Using a recursive coloring argument, they show that $O(m\log n)$ edge-detection queries are sufficient to learn an arbitrary hidden graph.
Reyzin and Srivastava~\cite{reyzin-srivastava-Learning2007} consider and compare the shortest path, edge-detection and edge-counting queries and prove a variety of lower and upper bounds for learning partitions, trees and arbitrary graphs.
In other related work, Beerliova et. al.~\cite{beerliova2006network} consider an oracle called the layered-graph oracle.

\subsection{Our Results}
\subsubsection{Membership queries}
Every graph $G$ admits a partition of its vertex set into connected components.
We say that a learner \emph{learns the components of $G$} if it learns this partition.
For an $n$-vertex hidden graph $G$ with $k$ components, one of the problems studied in \cite{reyzin-srivastava-Learning2007} entails learning the components of $G$.
More precisely, if $\alpha$ denotes the membership query given by
\[
\alpha(u,v)=\left\{
\begin{array}{rl}
	1, &\mbox{if }u\mbox{ and }v\mbox{ belong to the same component},\\
	0, &\mbox{otherwise};
\end{array}
\right.
\]
then they demonstrate an algorithm to learn the components of $G$.

\begin{thm-algo}[Reyzin and Srivastava~\cite{reyzin-srivastava-Learning2007}] \label{algo:reyzinSrivastavaMembership}
	\begin{itemize}
		\item Start with a set $S$ of already classified vertices and a set $C$ of component representatives which are both initialized to $\{v_0\}$ for some arbitrary $v_0\in V(G)$.
		\item For every unclassified vertex $v\in V(G)\setminus S$, sequentially query $\{\alpha(x,v):x\in C\}$.
		\item If some query $\alpha(x,v)$ is true, add $v$ to $S$. Otherwise, the component of $v$ has not been discovered till now, so add $v$ to both $S$ and $C$.
		\item Repeat the procedure until $S=V(G)$.
	\end{itemize}
\end{thm-algo}
It can be seen that when $k$ is small compared to $n$, Algorithm~\ref{algo:reyzinSrivastavaMembership} uses $O(nk)$ many $\alpha$-queries.
In particular, if we consider a graph $G$ which is the union of isolated vertices $\{1,\ldots, k-1\}$ and a clique $K_{\{k,\ldots, n\}}$, Proposition~\ref{prop:membershipAlgoUpperBd} demonstrates that Algorithm~\ref{algo:reyzinSrivastavaMembership} uses $(k-1)n -\binom{k}{2}$ queries in the worst case.

\medskip
However, the best known lower bound for this fundamental problem of learning components with $\alpha$ was an information-theoretic bound of $\Omega(n\log k)$ queries.
To the knowledge of the authors, no proof of an $\Omega(nk)$ lower bound has been published till date since it was first posed by Reyzin and Srivastava~\cite{reyzin-srivastava-Learning2007}.
Our main theorem is the first exact lower bound of $(k-1)n -\binom{k}{2}$ queries to this problem, when $k$ is known to the algorithm.
\begin{thm}
	\label{thm:membershipLowerBd}
	Given any algorithm $\mathcal A$ that makes membership queries on a hidden graph $G$ with $n$ vertices and $k$ components, there is an adversary that can force $\mathcal A$ to make at least {$(k-1)n -\binom{k}{2}$} queries to learn the partition of $G$.
\end{thm}

{We also prove an exact bound if the number of connected components in $G$ is unknown.}
\begin{thm}
	\label{thm:membershipLowerBd-k-unknown}
	{
		Given any algorithm $\mathcal A$ that makes membership queries on a hidden graph $G$ with $n$ vertices,
		there is an adversary that can force $\mathcal A$ to make at least {$kn-\binom{k+1}{2}$} queries to learn the partition of $G$,
		where $k$ is the number of connected components in $G$.
	}
\end{thm}
\begin{rmk}
	{
		When the number of components $k$ of a hidden graph $G$ is known to an algorithm, intuitively it should require fewer queries to learn all components of $G$.
		This is also reflected in the fact that our lower bound in Theorem~\ref{thm:membershipLowerBd} is $(n-k)$ lower than the lower bound in Theorem~\ref{thm:membershipLowerBd-k-unknown}.
		According to Proposition~\ref{prop:membershipAlgoUpperBd}, the bounds from both Theorems \ref{thm:membershipLowerBd} and \ref{thm:membershipLowerBd-k-unknown} are tight.
	}
\end{rmk}

We make a note here that Theorems~\ref{thm:membershipLowerBd} and \ref{thm:membershipLowerBd-k-unknown}, as well as the membership oracle is applicable to a much more general discrete setting: \emph{learning partitions of any finite $n$-element set into $k$ parts}.
However, the problem of learning partitions is not well-represented in the literature and hence we use the more familiar and equivalent language of graph learning for presenting Theorems~\ref{thm:membershipLowerBd} and \ref{thm:membershipLowerBd-k-unknown}.



\subsubsection{Multiple-membership queries}
Reyzin and Srivastava~\cite{reyzin-srivastava-Learning2007} also posed a question on whether there is an oracle that has to be queried fewer times to learn the number of components in an $n$-vertex hidden graph than learning the components.
For a vertex $u$ and a set $S$ of vertices not containing $u$, we define the multiple-membership query
\[
\alpha_m(u,S)=\left\{\begin{array}{rl}
	1, & \mbox{if }u\mbox{ and }v\mbox{ belong to the same component for some }v\in S,\\
	0, & \mbox{otherwise}.
\end{array}
\right.
\]
Our second result gives a positive answer to their question:
\begin{thm}
	\label{thm:alpha_mLearningComponents}
	For an $n$-vertex hidden graph $G$, learning the number of components of $G$ can be done using $O(n)$ $\alpha_m$-queries. However, learning all the components requires $\Theta(n\log k)$ $\alpha_m$-queries.
\end{thm}

\subsubsection{Vertex-neighborhood detection queries}
We modify $\alpha_m$ into a query that is more applicable in practical scenarios, which we call the vertex-neighborhood detection query $\beta$ given by:
\[
\beta(u,S)=\left\{\begin{array}{rl}
	1, & \mbox{if }u\mbox{ and }v\mbox{ are adjacent for some }v\in S,\\
	0, & \mbox{otherwise}.
\end{array}
\right.
\]

$\beta$ is useful in the setting of biochemistry where it is easy to detect a reaction between a fixed reagent and a set of other chemicals, and could be applicable to genome sequencing using polymerase chain reaction (PCR) tests~\cite{bouvel2005combinatorial,chang2011reconstruction}.

We analyze the problems of graph learning and graph verification using $\beta$-queries, and prove tight bounds for these problems (up to logarithmic factors):
\begin{thm}
	\label{thm:betaGraphLearning}
	Learning or verifying an $n$-vertex hidden graph on $m$ edges requires $\Omega(m)$ $\beta$-queries. Conversely, learning such a graph can be done in $O(m\log n)$ queries, whereas verifying can be done in $O(m+n)$ queries.
\end{thm}
\begin{rmk}
	As an immediate consequence to Theorem~\ref{thm:betaGraphLearning}, note that all sparse families of graph which satisfy $m=O(n)$, such as trees, planar graphs and minor-free graphs, can be both learned and verified using $\widetilde\Theta(n)$ queries to $\beta$.
\end{rmk}

This paper is organized as follows. In Section~\ref{sec:graphPartitions}, we prove Theorems~\ref{thm:membershipLowerBd}, \ref{thm:membershipLowerBd-k-unknown} and \ref{thm:alpha_mLearningComponents}. Section~\ref{sec:betaQuery} presents our results on the $\beta$ oracle, and proves Theorem~\ref{thm:betaGraphLearning}.
Finally, we make some concluding remarks in Section~\ref{sec:conclusion}.
\section{Graph partitions and the membership query}
\label{sec:graphPartitions}
In this section, we consider the membership query given by $\alpha(u,v)=1$ iff $u$ and $v$ belong to the same connected component. Our goal in this section is to prove Theorem~\ref{thm:membershipLowerBd}.

\subsection{Preliminaries}

We briefly state some preliminary results before diving into the proof.
\subsubsection{Upper bound on query complexity of Algorithm~\ref{algo:reyzinSrivastavaMembership}}
First, we give a quick demonstration of the upper bound of $(k-1)n-\binom k2$ of the query complexity of Algorithm~\ref{algo:reyzinSrivastavaMembership}.
\begin{prop}
	\label{prop:membershipAlgoUpperBd}
	Algorithm~\ref{algo:reyzinSrivastavaMembership} uses at most $(k-1)n-\binom k2$ queries to learn the graph $G$ consisting of isolated vertices $\{1,\ldots, k-1\}$ and a clique $K_{\{k,\ldots, n\}}$. If $k$ is not known to the algorithm, it requires $(n-k)$ additional queries.
\end{prop}
\begin{proof}
	Suppose $k$ is known. Without loss of generality assume that the algorithm determines the components of the $k-1$ isolated vertices using $1+\cdots+(k-2)=\frac 12(k-2)(k-1)$ many membership queries first.
	Otherwise, if the large component is discovered earlier, our algorithm can learn $G$ using fewer queries.
	To classify the remaining $(n-k+1)$ vertices into the $k$'th component, we only need to make sure that each of them does not share a component with $\{1,\ldots, k-1\}$.
	This requires $(k-1)(n-k+1)$ membership queries.
	Hence, the total number of $\alpha$-queries required in this case equals
	\[\frac12(k-1)(k-2) + (k-1)(n-k+1) = (k-1)\left(n-\frac k2\right) = (k-1)n -\binom{k}{2},\]
	as desired.
	
	When $k$ is not known, the last $(n-k+1)$ vertices each require not only $(k-1)$ queries each, but the vertices $\{k+1,\ldots, n\}$ each need one additional query.
	The total number of queries made therefore, is $(n-k)$ more, proving the second part of our claim.
\end{proof}

\subsubsection{Uniquely $k$-colorable graphs}
Our proof will be closely related to the following definition in graph theory.

Given a graph $G$, we say that it is $k$-colorable if there is a labelling $\chi:V(G)\to \{1,\ldots, k\}$ such that the two endpoints of any edge receives distinct labels.
Such a $\chi$ is also called a proper $k$-coloring.
$G$ is said to be \emph{uniquely} $k$-colorable if it has a unique proper $k$-coloring $\chi$ (up to permutations of $\{1,\ldots, k\}$).
The following theorem gives a lower bound for the number of edges in a uniquely $k$-colorable graph on $n$ vertices. 

\begin{thm}[\cite{truszczynski-UniqueColorable1984,shaoji-UniqueColorable1990}]
	\label{thm:unique-kCol}
	{
		A uniquely $k$-colorable graph on $n$ vertices must have at least $(k-1)n - \binom k2$ edges.
	}
\end{thm}

For the sake of completeness, we give an outline of the proof. We urge the reader to refer to the cited sources for further information on uniquely $k$-colorable graphs.
\medskip

\begin{proof}(Sketch).
	Suppose $G$ is a uniquely $k$-colorable graph on $n$ vertices, and its color classes are $I_1,\ldots,I_k$.
	Each $I_i$ is an independent set, and the edges of $G$ go across these color classes.
	Let us fix any pair $i\neq j$, and consider the bipartite graph $G[I_i\cup I_j]$.
	If $G[I_i\cup I_j]$ was disconnected, we could switch around the colors of these connected components to obtain a new proper $k$-coloring of $G$, a contradiction.
	Thus, $G[I_i\cup I_j]$ is a connected bipartite graph, implying there are at least $|I_i|+|I_j|-1$ edges between components $I_i$ and $I_j$.
	It can then be seen that
	\[
	\begin{aligned}
		|E(G)|&\ge \sum_{i<j} |I_i|+|I_j|-1 \\&= (k-1)(|I_1|+\cdots+|I_k|) - \binom k2 \\&= (k-1)n - \binom k2.
	\end{aligned}
	\]
\end{proof}

We are ready to present the proofs of Theorems~\ref{thm:membershipLowerBd} and \ref{thm:membershipLowerBd-k-unknown} in the following two sections. 

\subsection{Membership query}
First we present our proof of Theorem~\ref{thm:membershipLowerBd}.
Recall that $\alpha(u,v)=1$ iff $u$ and $v$ belong to the same component in an $n$-vertex hidden graph with $k$ components, and we are aiming to prove a lower bound of $(k-1)n - \binom k2$ queries on learning all its components.

\subsubsection{Proof of Theorem~\ref{thm:membershipLowerBd}.}
{We will use the following definition in our proofs.}
\begin{defn}
	{
		Suppose that $G$ is a $k$-colorable graph and $i,j \in V(G)$ are distinct vertices that are not adjacent in $G$.
		Then $\{i,j\}$ is called a $k$-separable pair if there exists a
		proper $k$-coloring $\chi\colon V(G)\to [k]$ of $G$ such that $\chi(i)\neq \chi(j)$,
		and we call such a coloring $\chi$ an $\{i,j\}$-separating coloring.
		Otherwise, we call $\{i,j\}$ a $k$-inseparable pair.
	}
\end{defn}
{For example, the nonadjacent pair in $K_{k+1}^{-}$ (the graph obtained from a $(k+1)$-clique by removing one edge) is $k$-inseparable.}
The following fact is easy to observe from the definition. 

\begin{fact}
	Suppose that $G$ is a $k$-colorable graph and $i,j \in V(G)$ are distinct vertices that are not adjacent in $G$.
	Then $\{i,j\}$ is a $k$-inseparable pair if and only if the graph $G'$ obtained from $G$ by adding the edge $\{i,j\}$ has chromatic number $k+1$.
\end{fact}

\medskip

Now we present our proof of Theorem~\ref{thm:membershipLowerBd}.
As mentioned earlier, our proof is adversarial.
We (the adversary) start with initializing an empty auxiliary (simple) graph $H$ with $|V(H)|=n$, and pick an arbitrary $k$-coloring $\chi:V(H) \to \{1,\ldots, k\}$ of $H$.
Let $G$ be the hidden graph corresponding to $H$, where each color class corresponds to a partition.

Suppose now that $\mathcal A$ makes a query $\alpha(x,y)$. We respond to $\mathcal A$ and update the graph $H$ and its $k$-coloring $\chi$ according to the following rules:
\begin{itemize}
	\item If $\chi(x)\neq \chi(y)$, we add $xy$ to $E(H)$ and reply ``no" to the algorithm.
	\item If $\chi(x)=\chi(y)$ and $\{x,y\}$ is $k$-separable in $H$,
	we add $xy$ to $E(H)$ and modify the coloring of $G$ to a $\{x,y\}$-separating $k$-coloring, and reply ``no".
	\item If $\chi(x)=\chi(y)$ and $\{x,y\}$ is $k$-inseparable,
	then we do not add $xy$ to $E(H)$, and answer ``yes".
\end{itemize}

We illustrate some intermediate steps in the evolution of $H$ and $\chi$ against a sample algorithm $\mathcal A$ in Figure~\ref{fig:coloringUpdates}.
\begin{figure}[ht]
	\resizebox{\textwidth}{!}{
		\begin{tikzpicture}
			\tikzstyle{redvertex}=[align=center, inner sep=0pt, text centered, circle,black,fill=red!10,draw=red,text width=1.3em,very thick]
			\tikzstyle{greenvertex}=[align=center, inner sep=0pt, text centered, circle,black,fill=green!10,draw=green,text width=1.3em,very thick]
			\tikzstyle{bluevertex}=[align=center, inner sep=0pt, text centered, circle,black,fill=blue!10,draw=blue,text width=1.3em,very thick]
			\begin{scope}
				\node [bluevertex] (x4) at (0:2){$4$};
				\node [redvertex] (x3) at (60:2){$3$};
				\node [redvertex] (x2) at (120:2){$2$};
				\node [redvertex] (x1) at (180:2){$1$};
				\node [greenvertex] (x6) at (240:2){$6$};
				\node [bluevertex] (x5) at (300:2){$5$};
				\draw (x3)--(x4)--(x2)--(x6)--(x2)--(x5);
				\draw (x1)--(x6)--(x3);
				\draw (270:3) node {$G:\ \{1,2,3\}, \{4,5\}, \{6\}$};
			\end{scope}
			
			\draw (3.5,0) node [above] {$ \alpha(1,5)$};
			\draw [->] (2.7,0)--(4.3,0);
			
			\begin{scope}[shift={(7,0)}]
				\node [bluevertex] (x4) at (0:2){$4$};
				\node [redvertex] (x3) at (60:2){$3$};
				\node [redvertex] (x2) at (120:2){$2$};
				\node [redvertex] (x1) at (180:2){$1$};
				\node [greenvertex] (x6) at (240:2){$6$};
				\node [bluevertex] (x5) at (300:2){$5$};
				\draw (x3)--(x4)--(x2)--(x6)--(x2)--(x5);
				\draw (x1)--(x6)--(x3);
				\draw (x1)--(x5);
				\draw (270:3) node {$G:\ \{1,2,3\}, \{4,5\}, \{6\}$};
			\end{scope}
			
			\draw (10.5,0) node [above] {$ \alpha(2,3)$};
			\draw [->] (9.7,0)--(11.3,0);
			
			\begin{scope}[shift={(14,0)}]
				\node [bluevertex] (x4) at (0:2){$4$};
				\node [redvertex] (x3) at (60:2){$3$};
				\node [redvertex] (x2) at (120:2){$2$};
				\node [redvertex] (x1) at (180:2){$1$};
				\node [greenvertex] (x6) at (240:2){$6$};
				\node [bluevertex] (x5) at (300:2){$5$};
				\draw (x3)--(x4)--(x2)--(x6)--(x2)--(x5);
				\draw (x5)--(x1)--(x6)--(x3);
				\draw [dashed] (x2)--(x3) node [above, midway] {\footnotesize separable};
				\draw (270:3) node {$G:\ \{1,2,3\}, \{4,5\}, \{6\}$};
			\end{scope}
			\draw [->] (14,-3.8)--(14,-4.3);
			
			\begin{scope}[shift={(14,-7)}]
				\node [greenvertex] (x4) at (0:2){$4$};
				\node [bluevertex] (x3) at (60:2){$3$};
				\node [redvertex] (x2) at (120:2){$2$};
				\node [redvertex] (x1) at (180:2){$1$};
				\node [greenvertex] (x6) at (240:2){$6$};
				\node [bluevertex] (x5) at (300:2){$5$};
				\draw (x3)--(x4)--(x2)--(x6)--(x2)--(x5);
				\draw (x5)--(x1)--(x6)--(x3);
				\draw (x2)--(x3);
				\draw (270:3) node {$G:\ \{1,2\}, \{3,5\}, \{4,6\}$};
			\end{scope}
			
			\draw (10.5,-7) node [above] {$ \alpha(1,3)$};
			\draw [->] (11.3,-7)--(9.7,-7);
			
			\begin{scope}[shift={(7,-7)}]
				\node [greenvertex] (x4) at (0:2){$4$};
				\node [bluevertex] (x3) at (60:2){$3$};
				\node [redvertex] (x2) at (120:2){$2$};
				\node [redvertex] (x1) at (180:2){$1$};
				\node [greenvertex] (x6) at (240:2){$6$};
				\node [bluevertex] (x5) at (300:2){$5$};
				\draw (x3)--(x4)--(x2)--(x6)--(x2)--(x5);
				\draw (x5)--(x1)--(x6)--(x3);
				\draw (x2)--(x3)--(x1);
				\draw (270:3) node {$G:\ \{1,2\}, \{3,5\}, \{4,6\}$};
			\end{scope}
			
			\draw (3.5,-7) node [above] {$\alpha(1,2)$};
			\draw [->] (4.3,-7)--(2.7,-7);
			
			\begin{scope}[shift={(0,-7)}]
				\node [greenvertex] (x4) at (0:2){$4$};
				\node [bluevertex] (x3) at (60:2){$3$};
				\node [redvertex] (x2) at (120:2){$2$};
				\node [redvertex] (x1) at (180:2){$1$};
				\node [greenvertex] (x6) at (240:2){$6$};
				\node [bluevertex] (x5) at (300:2){$5$};
				\draw (x3)--(x4)--(x2)--(x6)--(x2)--(x5);
				\draw (x1)--(x6)--(x3);
				\draw (x5)--(x1)--(x3);
				\draw (x2)--(x3);
				\draw [dashed] (x1)--(x2) node [midway, above, rotate=60] {\footnotesize inseparable};
				\draw (270:3) node {$G:\ \{1,2,3\}, \{4,5\}, \{6\}$};
			\end{scope}
		\end{tikzpicture} 
	}
	\caption{A sample set of queries and the corresponding evolution of $H$ and the partition of $G$. Here $n=6$ and $k=3$. Our adversary responds ``yes" only to $\alpha(1,2)$, and ``no" to all other $\alpha$-queries made by $\mathcal A$ in this example.}
	\label{fig:coloringUpdates}
\end{figure}
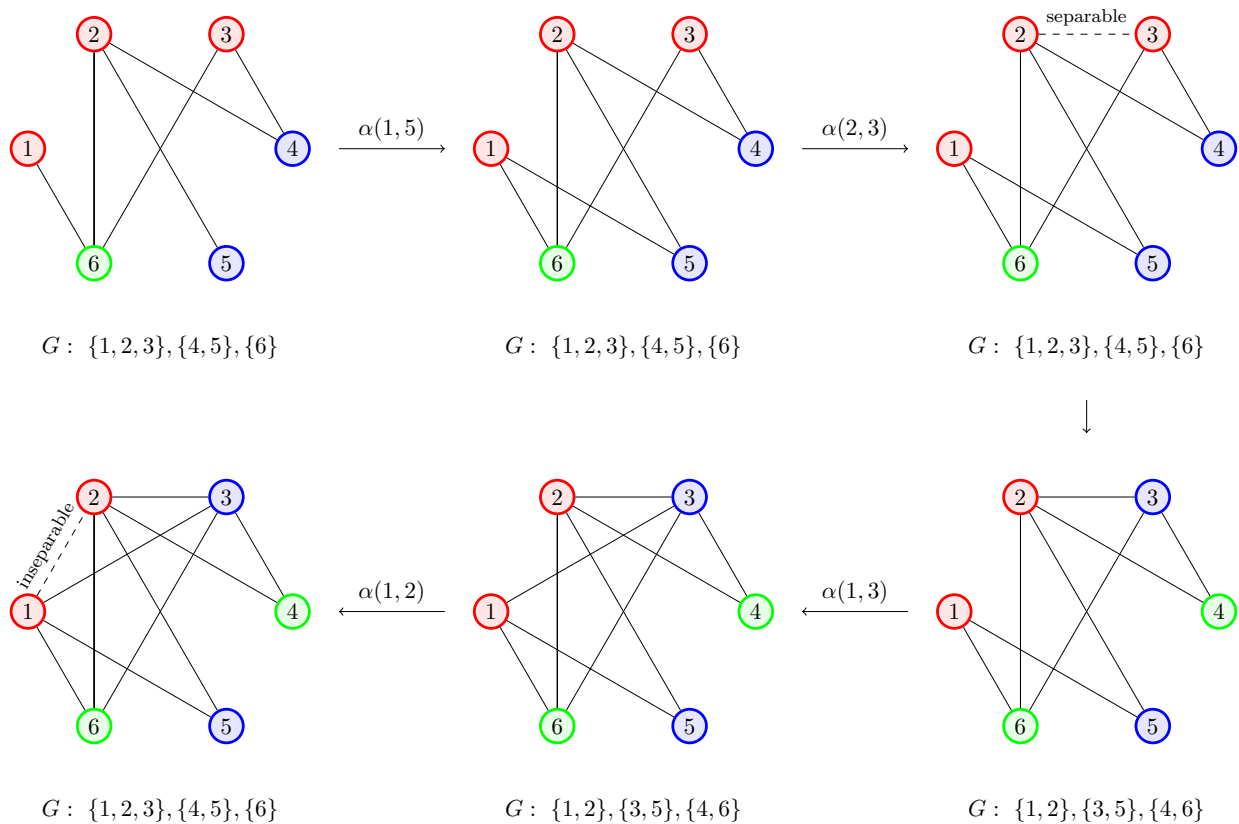

Observe that adding new edges into $H$ does not change the inseparability of a pair $\{x,y\}$.
Therefore, all $k$-inseparable pairs $\{x,y\}$ of $H$ remain $k$-inseparable in the evolution of the algorithm.
Further, by construction, every edge of $H$ corresponds to a non-coincidence in $G$'s components.
Vertex pairs $\{x,y\}$ of $H$ that have $\alpha(x,y)=1$ are $k$-inseparable, and hence belong to the same color class for any proper $k$-coloring of $H$.
Further, $\chi$ is always a proper $k$-coloring of $H$, and hence induces a partition of $G$ into $k$ components.

Now, suppose that $\mathcal A$ learns a unique partition of $G$ into $k$ components.
This would mean that the coloring $\chi$ of the graph $H$ must be a unique $k$-coloring.
By Theorem~\ref{thm:unique-kCol}, $H$ has at least $(k-1)n - \binom k2$ edges, as desired.
\hfill{$\Box$}

\begin{rmk}
	We note that our adversary checks over all possible $k$-colorings of $H$, and hence does not have a $\poly(n,k)$ runtime.
	However, a slight modification to the argument above (given in Appendix~\ref{app:degreeLowerBd}) is able to prove the existence of a $\poly(n,k)$-time adversary that forces at least $\frac12(n-k)(k-1)$ queries.
	It might be interesting to analyze the effect of limiting resources available to the adversary to prove different versions of this theorem.
\end{rmk}

\subsubsection{Proof of Theorem~\ref{thm:membershipLowerBd-k-unknown}}
Now we prove Theorem~\ref{thm:membershipLowerBd-k-unknown}.
Our proof is basically a modification of the proof of Theorem~\ref{thm:membershipLowerBd}: instead of a single auxiliary graph $H$ with $|V(H)|=n$, now we will maintain two graphs $H_1$ and $H_2$ on the same vertex set $V$ of size $n$.
We start with the coloring $\chi\colon V\to \{1\}$, i.e. all vertices of $V$ receive the color $1$.\\
Suppose an algorithm $\mathcal A$ makes a query $\alpha(x,y)$. We then respond according to the rules below:
\begin{itemize}
	\item If $\chi(x)\neq \chi(y)$, add $xy$ to $E(H_1)$ and reply ``no" to the algorithm.
	\item If $\chi(x)=\chi(y)$ and $\{x,y\}$ is $k$-separable in $H_1$,
	add $xy$ to $E(H_1)$ and modify the coloring of $H_1$ to a $\{x,y\}$-separating coloring, and reply ``no".
	\item If $\chi(x)=\chi(y)$ and $\{x,y\}$ is $k$-inseparable in $H_1$,
	add $xy$ to $E(H_2)$ and answer ``yes".
\end{itemize}
Note that similar to before, adding new edges into $H_1$ does not change the inseparability of a pair $\{x,y\}$. Hence, if $\mathcal A$ learns the hidden graph $G$ when the algorithm stops, the graph $H_1$ must be a uniquely $k$-colorable graph.
This implies
\[
|E(H_1)|\ge (k-1)n-\binom k2.
\]

Additionally, we claim the following about $H_2$.
\begin{clm}
	When $\mathcal A$ stops, the graph $H_2$ must have at most $k$ connected components.
\end{clm}
\begin{proof}
	Let us assume that the unique $k$-partition corresponding to $H_1$ is $\mathcal{P}_k = \{V_1,\ldots, V_k\}$.
	Suppose on the contrary that $H_2$ has $m$ connected components and $m\ge k+1$, say with vertex sets $C_1, \ldots, C_m$.
	As every edge of $H_2$ corresponds to an inseparable pair, for every $i\in [m]$, there exists $j\in [k]$ such that $C_i \subset V_j$.
	In other words, the $m$-partition $\mathcal{P}_m = \{C_1,\ldots, C_m\}$ is a refinement of $P_{k}$.
	The contradiction is that $\mathcal A$ cannot distinguish
	the $m$-partition $\mathcal{P}_m$ from the $k$-partition $\mathcal{P}_k$.
	Therefore, $H_2$ has at most $k$ connected components.
\end{proof}

As an immediate corollary, we obtain $|E(H_2)| \ge n-k$.
Therefore, the algorithm requires at least 
\[|E(H_1)|+|E(H_2)| \ge (k-1)n-\binom{k}{2}+n-k = kn-\binom{k+1}2\]
many queries, completing the proof of Theorem~\ref{thm:membershipLowerBd-k-unknown}.


\medskip

\subsection{Multiple-membership query}
Now we turn our attention to the problem of learning both $k$ and individual components using multiple-membership queries.

\subsubsection{Proof of Theorem~\ref{thm:alpha_mLearningComponents}}
Recall that $\alpha_m(u,S)=1$ if and only if there is a $v\in S$ such that $u$ and $v$ belong to the same component. There are three assertions in Theorem~\ref{thm:alpha_mLearningComponents}, and we prove each of them below.
\begin{itemize}
	\item {\bf Part 1. Learning $k$:} First, we demonstrate an algorithm that learns the number of components of an $n$-vertex hidden graph $G$ using $O(n)$ queries to $\alpha_m$.
	Start with $S=V(G)$.
	While there is a vertex $v\in S$ such that $\alpha_m(v,S\setminus\{v\})=1$, we delete $v$ from $S$.
	When this algorithm terminates, we end up with an independent set $S$ whose each vertex lies in a different component.
	Conversely, as every other vertex $v\in V(G)$ lies in the same component as some vertex in $G$, $|S|$ will equal the number of components of $G$.
	Therefore, the above algorithm learns $k$ using $O(n)$ many $\alpha_m$-queries. \hfill{$\blacksquare$}
	
	\item {\bf Part 2. Lower bound on learning components:} Next, we claim that learning all components of $G$ requires at least $\Omega(n\log k)$ queries.
	This proof is information-theoretic.
	The total number of partitions of $n$ vertices into $k$ parts is $\Omega(k^n)$, and hence any algorithm making $\alpha_m$-queries must have depth $\Omega(n\log k)$. \hfill{$\blacksquare$}
	
	\item {\bf Part 3. Upper bound on learning components:} Note that the function \ref{func:betacomponentLearning} learns the components of the hidden graph $G$ using multiple-membership queries.
	\begin{function}[ht]
		\SetKwInOut{Input}{Input}\SetKwInOut{Output}{Output}
		\Input{Vertex set $V(G)=\{v_1,\ldots,v_n\}$, Oracle $\alpha_m$.}
		\Output{Partition $V(G)=C_1\sqcup\cdots\sqcup C_k$ such that each $C_i$ is a connected component.}
		\caption{learnComponentsMMQ()\label{func:betacomponentLearning}}
		\BlankLine
		Initialize $C_1=\{v_1\}$, $k=1$\;
		\For{$i=2$ to $n$}{
			\If{$\alpha_m(v_i, C_1\cup\cdots \cup C_k)=0$}{
				k += 1 \;
				Add $v_i$ to $C_k$\;
			}
			\Else{
				Find $j$ such that $\alpha_m(v_i, C_j)=1$ via binary search among $\{C_1,\ldots, C_k\}$\label{step:binarySearchComp}\;
				Add $v_i$ to its corresponding $C_j$\;
			}
		}
		\Return{$\{C_1,\ldots, C_k\}$}
	\end{function}
	
	The binary search in the \textbf{else} statement works because $\alpha_m$ can query $v_i$ with a union of several $C_j$'s at once. This step requires $O(\log k)$ queries, and therefore our algorithm has query complexity of $O(n\log k)$.\hfill{$\blacksquare$}
\end{itemize}

This completes the proof of Theorem~\ref{thm:alpha_mLearningComponents}, and hence $\alpha_m$ is an oracle that can learn the number of components $k$ in a hidden graph with fewer queries than learning the components. \hfill{$\Box$}

\section{The vertex-neighborhood detection query}
\label{sec:betaQuery}

For the remainder of the paper, we consider the vertex-neighborhood detection query given by $\beta(v,S)=1$ iff there is some edge from $v$ to some vertex in $S$.
We now prove tight bounds of $\widetilde \Theta(m)$ for both learning and verifying graphs (with $m$ edges) using $\beta$-queries.

\subsection{Proof of Theorem~\ref{thm:betaGraphLearning}}
Now we move onto analyzing the problems of graph learning and verification using the oracle $\beta$.
Let us take a hidden graph $G$ on $n$ vertices and $m$ edges. It is clear that $\beta$ can only detect one edge at a time, and hence both learning and verifying a hidden graph using $\beta$ would trivially require at least $\Omega(m)$ queries.
Hence, it suffices to prove upper bounds of $O(m+n)$ for verification and $O(m\log n)$ for learning $G$.

\subsubsection{Verification using $\beta$}
For the verification problem, we have a graph $\widehat G$ with $V(\widehat G) = V(G) = V$ that is known to us. We verify each edge $uv\in E(\widehat G)$ individually by checking $\beta(u,\{v\})=1$, and this requires $m$ queries.

Next, we verify the non-edges of $\widehat G$. Fix a vertex $v$ and compute its neighborhood $N_{\widehat G}(v)=\{u\in V: uv\in E(\widehat G)\}$.
Note that if $\widehat G$ was the same as $G$, we would have
\[
\beta\Bigl(v, V\setminus ( N_{\widehat G}(v)\cup\{v\} )\Bigr)=0.
\]
Since checking all non-edges through a vertex takes a single $\beta$-query, we can verify all non-edges of $\widehat G$ using $n$ queries.
Hence, verifying all edges and non-edges of $G$ using can be done using $O(m+n)$ $\beta$-queries.\hfill{$\blacksquare$}

\subsubsection{Learning using $\beta$}
Although slightly more involved, our algorithm for learning a hidden graph $G$ is also based on divide-and-conquer.
The main step consists of devising a recursive method (which we call \ref{func:betaLearning}) that, for a fixed vertex $v\in V$, learns all neighbors (and non-neighbors) of $v$.

\begin{function}[ht]
	\SetKwInOut{Input}{Input}\SetKwInOut{Result}{Result}
	\Input{Vertex $v$, Set $S$, Hidden graph $G$, Oracle $\beta$.}
	\Result{Mark all neighbors of $v$ in $S$ blue, and non-neighbors red.}
	\caption{findNeighbors($v, S$).\label{func:betaLearning}}
	\BlankLine
	\If{$|S|=1$}{
		\If{$\beta(v,S)=1$}{
			Mark the vertex of $S$ blue and terminate\;
		}
		\Else{
			Mark the vertex of $S$ red and terminate\;
		}
	}
	Divide $S$ into two (approximately) equal parts $S_1\sqcup S_2$\;
	\If{$\beta(v, S_1)=1$}{
		findNeighbors($v, S_1$)\;
	}
	\Else{
		Mark all vertices of $S_1$ red and terminate\;
	}
	\If{$\beta(v, S_2)=1$}{
		findNeighbors($v, S_2$)\;
	}
	\Else{
		Mark all vertices of $S_2$ red and terminate\;
	}
\end{function}

Let us first analyze the query complexity of \ref{func:betaLearning}.

\begin{clm}
	\label{clm:findNeighborsQuery}
	Suppose $\ell$ is the number of neighbors of $v$ in $S$.
	Then, \ref{func:betaLearning}($v,S$) takes at most $O(\ell\log |S|)$ queries to mark all neighbors of $v$ in $S$ red and non-neighbors blue.
\end{clm}
\begin{proof}(Claim~\ref{clm:findNeighborsQuery}):
	We take a look at the recursion tree $T$ for \ref{func:betaLearning}($v,S$), and color its nodes red or blue as follows.
	The root node corresponds to the computation of \ref{func:betaLearning}($v,S$), and so we keep the set $S$ in it.
	If $\beta(v,S)=1$, we color the root node blue, otherwise we color it red.
	In the next level, the node containing $S$ has two children: $S_1$ and $S_2$ corresponding to the equipartition of $S$.
	Let us color $S_1$ blue if $\beta(v,S_1)=1$, and red otherwise. Similarly, we color $S_2$, and continue this coloring scheme down the entire recursion tree.
	A sample $T$ and its coloring is depicted in Figure~\ref{fig:recursionTree}.
	
	\begin{figure}[ht]
		\centering
		\includegraphics[width=.7\textwidth]{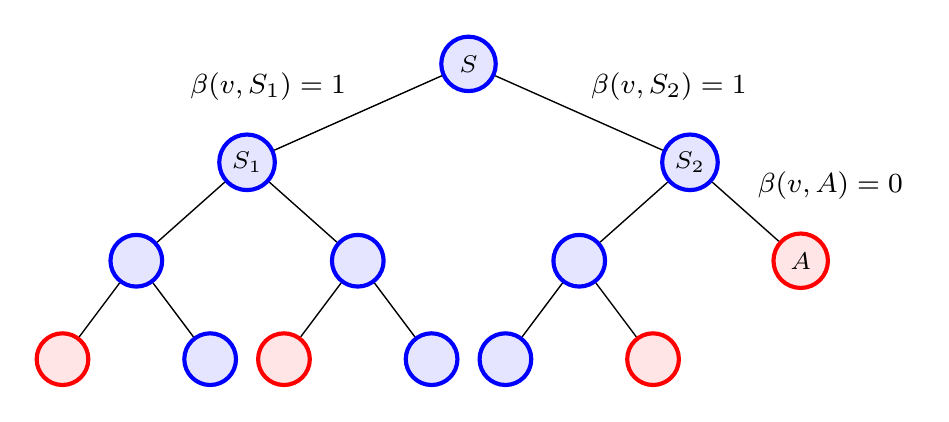}
		\caption{A sample recursion tree $T$ corresponding to \ref{func:betaLearning}($v,S$)}
		\label{fig:recursionTree}
	\end{figure}
	
	Let $T_b$ be the subtree with blue vertices. It is clear that each leaf of $T_b$ corresponds to a vertex marked blue by \ref{func:betaLearning}, implying that $T_b$ is a tree with $\ell$ leaves and depth at most $\log|S|$. Therefore,
	\[
	|V(T_b)|\le O(\ell\log |S|).
	\]
	
	Now we take a closer look at the red vertices of $T$.
	By construction, each red vertex is a leaf, i.e. the computation stops at these vertices.
	This means that the parents of red vertices are always blue.
	Further, two red nodes cannot be siblings, as if $\beta(v,A)=\beta(v,B)=0$, then $\beta(v, A\cup B)=0$, implying that the computation would have stopped at the parent of $A$ and $B$.
	Thus, every red vertex has a \emph{unique} blue parent.
	This implies that the number of red vertices in $T$ is bounded above by $|V(T_b)|$, leading to
	\[
	|V(T)|\le 2|V(T_b)| \le O(\ell \log |S|),
	\]
	completing the proof of Claim~\ref{clm:findNeighborsQuery}.
\end{proof}

Therefore, to learn $G$, we can run \ref{func:betaLearning}($v, V\setminus \{v\}$) over all vertices $v$. Observe that the total number of queries made is at most
\[
\sum_{v\in V(G)} O(\deg(v)\log n) = O(m\log n),
\]
thus proving the required upper bound. \hfill{$\blacksquare$}

\section{Conclusion and future work}
\label{sec:conclusion}
In this paper, we demonstrated a fundamental and exact lower bound on the problem of learning partitions using membership queries (which we called $\alpha$), filling a long lasting gap in the literature.
We generalized the membership oracle to take subsets of vertices as one of its inputs (which we called $\alpha_m$), and demonstrated how learning the number of components can be done using fewer $\alpha_m$-queries than the components themselves.
In the second section, we also demonstrated a powerful oracle (which we called $\beta$) that can be used to efficiently learn and verify sparse graphs.

It would be very interesting to see other oracles which can exploit structural properties of sparse graphs, such as the existence of small separators.
Graph families that admit the existence of small separators include a vast array of graphs such as planar graphs~\cite{lipton1979separator,alon1994planar}, bounded genus graphs~\cite{gilbert1984separator}, minor-free graphs~\cite{reed2009linear,wulff2011separator}, etc.
Further, it should be possible to extend these results to learning and verification of graphs with polynomially bounded expansion~\cite{nevsetvril2008grad,dvorak2016strongly}, and we leave this as a future avenue of investigation.

\section*{Acknowledgments}
The authors are immensely thankful to Mano Vikash Janardhanan for bringing the membership oracle to their attention, and for several fruitful discussions on graph learning theory.
Research of the first author was supported by ERC Advanced Grant 101020255.


\bibliographystyle{plain}
\bibliography{graphLearning.bbl}

\clearpage
\appendix
\section{Appendix: Polynomial-time adversary against Algorithm~\ref{algo:reyzinSrivastavaMembership}}
\label{app:degreeLowerBd}
Our goal in this section is to prove a slight modification of Theorem~\ref{thm:membershipLowerBd}. We prove that by sacrificing a constant factor of $\frac 12$, it is possible to demonstrate an adversary that runs in $\poly(n,k)$ time:
\begin{prop}
	Given any algorithm $\mathcal A$ that makes membership queries on a hidden graph $G$ with $n$ vertices and $k$ components, there is a $\poly(n,k)$-time adversary that can force $\mathcal A$ to make at least {$\frac12(n-k)(k-1)$} $\alpha$-queries to learn the partition of $G$.
\end{prop}

\begin{proof}
	As in the original proof, we initialize an empty graph $H$ with $|V(H)|=n$, and pick any $k$-coloring $\chi:V(H) \to \{1,\ldots, k\}$ of $H$.
	Let $G$ be the hidden graph corresponding to $H$, where each color class corresponds to a partition.
	For the sake of simplicity, we shall call vertices of $H$ with degree at least $k-1$ \emph{big}, and those with degree less than $k-1$ \emph{small}.
	
	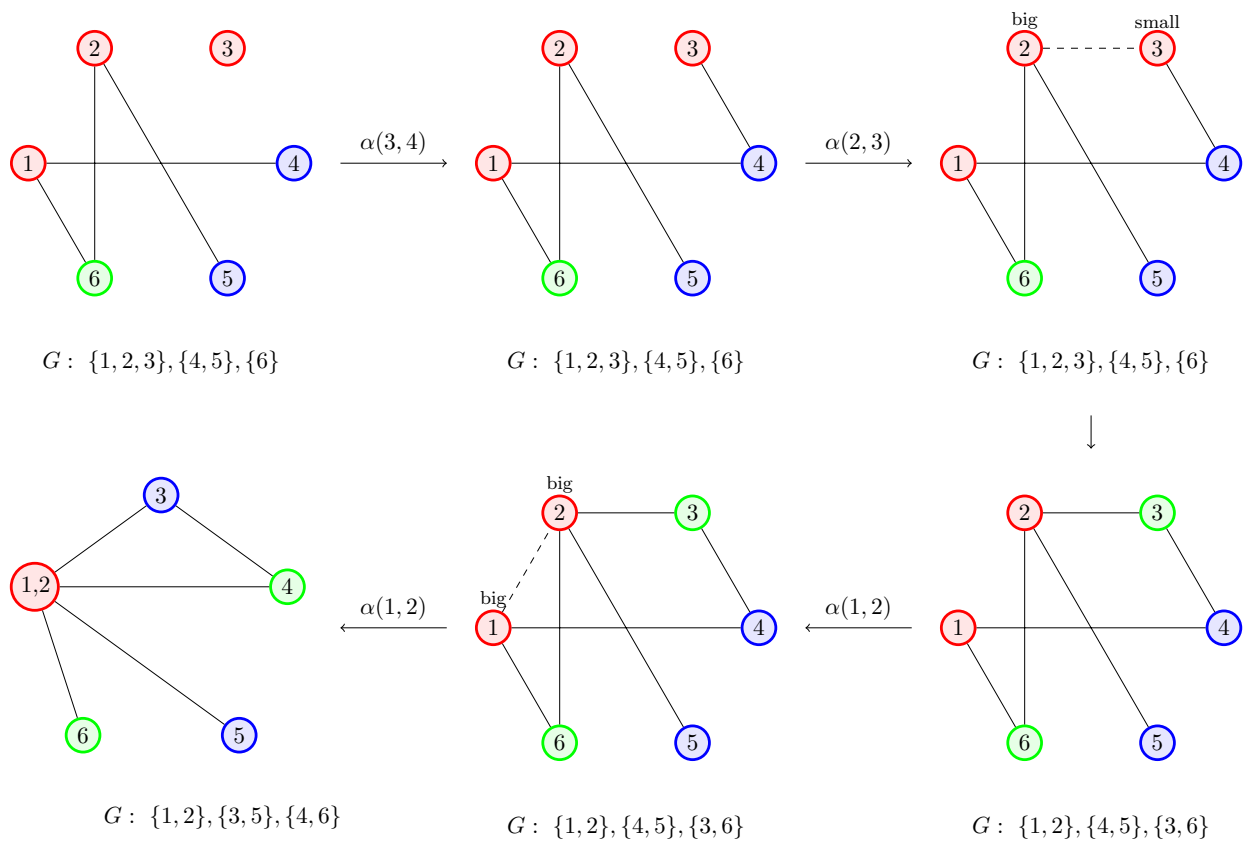
\begin{figure}[ht]
		\resizebox{\textwidth}{!}{
			\begin{tikzpicture}
				\tikzstyle{redvertex}=[align=center, inner sep=0pt, text centered, circle,black,fill=red!10,draw=red,text width=1.3em,very thick]
				\tikzstyle{bigredvertex}=[align=center, inner sep=2pt, circle,black,fill=red!10,draw=red,text width=1.2em,very thick]
				\tikzstyle{greenvertex}=[align=center, inner sep=0pt, text centered, circle,black,fill=green!10,draw=green,text width=1.3em,very thick]
				\tikzstyle{bluevertex}=[align=center, inner sep=0pt, text centered, circle,black,fill=blue!10,draw=blue,text width=1.3em,very thick]
				\begin{scope}
					\node [bluevertex] (x4) at (0:2){$4$};
					\node [redvertex] (x3) at (60:2){$3$};
					\node [redvertex] (x2) at (120:2){$2$};
					\node [redvertex] (x1) at (180:2){$1$};
					\node [greenvertex] (x6) at (240:2){$6$};
					\node [bluevertex] (x5) at (300:2){$5$};
					\draw (x4)--(x1)--(x6)--(x2)--(x5);
					\draw (270:3) node {$G:\ \{1,2,3\}, \{4,5\}, \{6\}$};
				\end{scope}
				
				\draw (3.5,0) node [above] {$ \alpha(3,4)$};
				\draw [->] (2.7,0)--(4.3,0);
				
				\begin{scope}[shift={(7,0)}]
					\node [bluevertex] (x4) at (0:2){$4$};
					\node [redvertex] (x3) at (60:2){$3$};
					\node [redvertex] (x2) at (120:2){$2$};
					\node [redvertex] (x1) at (180:2){$1$};
					\node [greenvertex] (x6) at (240:2){$6$};
					\node [bluevertex] (x5) at (300:2){$5$};
					\draw (x3)--(x4)--(x1)--(x6)--(x2)--(x5);
					\draw (270:3) node {$G:\ \{1,2,3\}, \{4,5\}, \{6\}$};
				\end{scope}
				
				\draw (10.5,0) node [above] {$ \alpha(2,3)$};
				\draw [->] (9.7,0)--(11.3,0);
				
				\begin{scope}[shift={(14,0)}]
					\node [bluevertex] (x4) at (0:2){$4$};
					\node [redvertex] (x3) at (60:2){$3$};
					\node [redvertex] (x2) at (120:2){$2$};
					\node [redvertex] (x1) at (180:2){$1$};
					\node [greenvertex] (x6) at (240:2){$6$};
					\node [bluevertex] (x5) at (300:2){$5$};
					\draw (x3)--(x4)--(x1)--(x6)--(x2)--(x5);
					\draw [dashed] (x2)--(x3);
					\draw (x2) node [above=.5em] {\footnotesize big};
					\draw (x3) node [above=.5em] {\footnotesize small};
					\draw (270:3) node {$G:\ \{1,2,3\}, \{4,5\}, \{6\}$};
				\end{scope}
				\draw [->] (14,-3.8)--(14,-4.3);
				
				\begin{scope}[shift={(14,-7)}]
					\node [bluevertex] (x4) at (0:2){$4$};
					\node [greenvertex] (x3) at (60:2){$3$};
					\node [redvertex] (x2) at (120:2){$2$};
					\node [redvertex] (x1) at (180:2){$1$};
					\node [greenvertex] (x6) at (240:2){$6$};
					\node [bluevertex] (x5) at (300:2){$5$};
					\draw (x3)--(x4)--(x1)--(x6)--(x2)--(x5);
					\draw (x2)--(x3);
					\draw (270:3) node {$G:\ \{1,2\}, \{4,5\}, \{3,6\}$};
				\end{scope}
				
				\draw (10.5,-7) node [above] {$ \alpha(1,2)$};
				\draw [->] (11.3,-7)--(9.7,-7);
				
				\begin{scope}[shift={(7,-7)}]
					\node [bluevertex] (x4) at (0:2){$4$};
					\node [greenvertex] (x3) at (60:2){$3$};
					\node [redvertex] (x2) at (120:2){$2$};
					\node [redvertex] (x1) at (180:2){$1$};
					\node [greenvertex] (x6) at (240:2){$6$};
					\node [bluevertex] (x5) at (300:2){$5$};
					\draw (x3)--(x4)--(x1)--(x6)--(x2)--(x5);
					\draw (x2)--(x3);
					\draw [dashed] (x1)--(x2);
					\draw (x1) node [above=.5em] {\footnotesize big};
					\draw (x2) node [above=.5em] {\footnotesize big};
					\draw (270:3) node {$G:\ \{1,2\}, \{4,5\}, \{3,6\}$};
				\end{scope}
				
				\draw (3.5,-7) node [above] {$\alpha(1,2)$};
				\draw [->] (4.3,-7)--(2.7,-7);
				
				\begin{scope}[shift={(0,-7)}, rotate=18]
					\node [greenvertex] (x4) at (0:2){$4$};
					\node [bluevertex] (x3) at (72:2){$3$};
					\node [bigredvertex] (x12) at (144:2){$1$,$2$};
					\node [greenvertex] (x6) at (216:2){$6$};
					\node [bluevertex] (x5) at (288:2){$5$};
					\draw (x3)--(x4)--(x12)--(x6);
					\draw (x3)--(x12)--(x5);
					\draw (270:3) node {$G:\ \{1,2\}, \{3,5\}, \{4,6\}$};
				\end{scope}
			\end{tikzpicture} 
		}
		\caption{A sample set of queries and the corresponding evolution of $H$ and the partition of $G$. Here $n=6$ and $k=3$, and vertices of degree at least $2$ are ``big". Our adversary responds ``yes" to $\alpha(1,2)$ and ``no" to every other query made by $\mathcal A$.}
		\label{fig:coloringUpdatesDegree}
	\end{figure}
	
	Suppose now that $\mathcal A$ makes a query $\alpha(x,y)$. We respond to $\mathcal A$ as follows:
	\begin{itemize}
		\item If $\chi(x)\neq \chi(y)$, we add $xy$ to $E(H)$ and reply ``no" to the algorithm.
		\item If $\chi(x)=\chi(y)$ and either $x$ or $y$ is \emph{small}, say $x$, then we note that $x$ has at most $k-2$ neighbors and hence two admissible colors.
		We then modify $\chi(x)$ to its other admissible color different from $\chi(y)$, add the edge $xy$ to $E(H)$, and reply ``no".
		\item Finally, if $\chi(x)=\chi(y)$ and both $x$ and $y$ are \emph{big}, we identify (i.e., contract) vertices $x$ and $y$.
		In terms of $G$, this would mean $x$ and $y$ belong to the same component.
	\end{itemize}
	
	We illustrate some intermediate steps in the evolution of $H$ and $\chi$ in Figure~\ref{fig:coloringUpdatesDegree}.

	Note that $\chi$ is always a proper $k$-coloring of $H$, and hence induces a partition of $G$ into $k$ components. Suppose that $\mathcal A$ learns a unique partition of $G$ into $k$ components.
	Then, there are two cases to consider:
	\begin{itemize}
		\item {\bf Case 1. The adversary contracts at least $\frac n2$ vertices:} 
		In this case, as each contracted vertex has degree at least $k-1$, $\mathcal A$ must have made at least $\frac n2(k-1)$ queries.
		\item {\bf Case 2. The adversary contracts less than $\frac n2$ vertices:}
		In this case, $\mathcal A$ can learn a unique partition of $G$ iff $\chi$ is a unique $k$-coloring of $H$. Since $H$ started with $n$ vertices, $H$ is a graph on at least $\frac n2$ vertices which is uniquely $k$-colorable. By Theorem~\ref{thm:unique-kCol}, $H$ has at least $\frac n2(k-1) - \binom k2$ edges.
	\end{itemize}
	
	In either case, $\mathcal A$ makes at least $\frac n2(k-1)-\binom k2=\frac12(n-k)(k-1)$ queries, as desired.
\end{proof}

\end{document}